\documentclass{acm_proc_article-sp}

\usepackage{xfrac}
\usepackage{courier}
\usepackage{graphicx}
\usepackage{amsmath}
\usepackage{cite}
\usepackage{epstopdf}
\usepackage{verbatim}
\usepackage{changebar}

\begin{document}

\title{Decision Tree Classification on Outsourced Data}


\numberofauthors{2} %  in this sample file, there are a *total*
\author{
% 1st. author
\alignauthor
Koray Mancuhan\titlenote{The author was visiting Qatar University
when this work was prepared.}\\
       \affaddr{Purdue University}\\
       \affaddr{305 N University St}\\
       \affaddr{West Lafayette, IN 47906}\\
       \email{kmancuha@purdue.edu}
% 2nd. author
\alignauthor
Chris Clifton\\
       \affaddr{Purdue University}\\
       \affaddr{305 N University St}\\
       \affaddr{West Lafayette, IN 47906}\\
       \email{clifton@cs.purdue.edu}
}

% There's nothing stopping you putting the seventh, eighth, etc.
% author on the opening page (as the 'third row') but we ask,
% for aesthetic reasons that you place these 'additional authors'
% in the \additional authors block, viz.
%\additionalauthors{Additional authors: John Smith (The Th{\o}rv{\"a}ld Group,
%email: {\texttt{jsmith@affiliation.org}}) and Julius P.~Kumquat
%(The Kumquat Consortium, email: {\texttt{jpkumquat@consortium.net}}).}
% Just remember to make sure that the TOTAL number of authors
% is the number that will appear on the first page PLUS the
% number that will appear in the \additionalauthors section.

\maketitle
\begin{abstract}
This paper proposes a client-server decision tree learning method for 
outsourced private data. The privacy model is anatomization/fragmentation:  
the server sees data values, but
the link between sensitive and identifying information is encrypted
with a key known only to clients. Clients have limited processing
and storage capability. Both sensitive and identifying information
thus are stored on the server.  The approach presented also retains 
most processing at the server, and client-side processing is amortized 
over predictions made by the clients.
Experiments on various datasets show that the method produces
decision trees approaching the accuracy of a non-private decision
tree, while substantially reducing the client's computing 
resource requirements.
\end{abstract}

%Handle this later

 % NOT required for Proceedings

\newdef{definition}{Definition}
\newtheorem{theorem}{Theorem}

\section{Introduction}
\label{sec:intro}
Data publishing without revealing sensitive information is an 
important problem. Many privacy definitions have been proposed
based on generalizing/suppressing data ($l$-diversity\cite{ldiversity}, 
k-anonymity \cite{kanon_defn} \cite{kanon_sweeney}, t-closeness 
\cite{tcloseness}).  Other alternatives include value swapping
\cite{CensusRelease} and noise addition (e.g., differential
privacy \cite{DifferentialPrivacy}).
% have been proposed in last decade to achieve a 
%secure data publishing where individuals cannot be linked to 
%the sensitive information. The main stream approaches to satisfy
%these privacy definitions are generalization and suppression 
%(anonymization methods) \cite{ref5} \cite{ref6}.
Generalization consists of replacing identifying attribute values
with a less specific version \cite{k_anon_survey}. Suppression can be
viewed as the ultimate generalization, replacing the identifying
value with an ``any'' value.
%of protecting sensitive information by removing it
\cite{k_anon_survey}.
These approaches have the advantage of preserving truth, but
a less specific truth that reduces the
utility of the published data.

Xiao and Tao proposed anatomization as a
method to enforce $l$-diversity while preserving specific
data values\cite{XiaoAnatomy}. Anatomization
splits  tuples across two tables, one containing identifying information 
and the other containing private information.
The more general approach of fragmentation \cite{fragmentation}
divides a given dataset's attributes into two
sets of attributes (2 partitions) such that an encryption mechanism avoids 
associations between two different small partitions. Vimercati et al. extend 
fragmentation to multiple partitions \cite{loose_fragmentation},
and Tamas et al. propose an extension that
deals with multiple sensitive attributes \cite{l_diver_multiple}.
The main advantage of anatomization/fragmentation
is that it preserves the original
values of data; the uncertainty is only in the mapping between individuals and sensitive values.
%identifier attributes and sensitive attributes; ; and that it 
%simultaneously satisfies ${\frac{1}{l}}$ constraint for l-diversity 
%in terms of individually identifiable private information. However, 
%anatomization has two shortcomings. Firstly, Xiao and Tao deals 
%only protection of a single private attribute.

There is concern that anatomization is vulnerable to 
several prior knowledge attacks \cite{Definetti}, \cite{AnatomyImprov}.
While this can be an issue, \emph{any} method that provides
any meaningful utility fails to provide perfect privacy
against a sufficiently strong adversary \cite{DifferentialPrivacy}.
Legally recognized standards such as the so-called ``Safe Harbor''
rules of the U.S. Healthcare Insurance Portability and Accountability
Act (HIPAA) \cite{HIPAA} allow release of data that bears some
risk of re-identification; the belief is that the greater
good of beneficial use of the data outweighs the risk to privacy.

This paper does not attempt to resolve that issue (although we
use a randomized rather than optimized grouping that provides
resistance to the type
of attacks above.) Instead,
this paper attempts to develop
a simple and efficient decision tree learning method from
an ${l}$-diverse dataset which satisfies ${l}$-diversity
in anatomization. 
The developed decision tree is
embedded into a recent client-server database model 
\cite{OutsourceUpdate,ErhanIFIP}. 

The database model
is a storage outsourcing mechanism that supports anatomization.
In this database model, the client is the owner of the data 
and stores the data at the server in an ${l}$-diverse anatomy scheme. 
The server provides some query preprocessing that reduce client's workload.
The model proposed by Nergiz et al. has four significant
aspects relevant to this paper  \cite{OutsourceUpdate,ErhanIFIP}: 1)
a given data table (or dataset) of a client is split into two partitions, 
a sensitive table (single sensitive attribute) and an
identifier table (multiple identifying or quasi-identifying attributes).
2)  A join key is stored with the data, but encrypted with a key known
only to the client, preventing the server from associating a sensitive
value with the correct identity.  3)  There is also an unencrypted
group-level join key allowing the server to map a group of identifiers
with a group of sensitive attributes, these groups are $l$-diverse.
4)  The server can use this group information to perform some query
processing (or in the case of this paper, tree learning), reducing
the client effort needed to get final results.
\begin{comment}
Splitting encrypts
the mapping between the tuples of identifying and sensitive
tables. Client is assumed to know the encryption key (called join
key) so that it can
decrypt this mapping to retrieve original data. Server is assumed to be
ignorant of client's join key.
Secondly, identifying and sensitive tables are 
stored on server. Thirdly, client (data table's owner) is capable
of executing several queries and update operations on the
split data with minimal processing. Most operations require
just decryption
of the encrypted mapping between few tuples of identifying
and sensitive tables. Irrelevant tuples of query results
are mostly eliminated by server. Fourth and last, client
is able to rebuilt the entire original data table by decrypting the
encrypted mappings using its join key created at the splitting step. 
\end{comment}
%We have 
%just given a brief summary of Nergiz et al.'s client-server database model for
%aspects relevant to this paper. 
The reader is advised to visit \cite{ErhanIFIP,OutsourceUpdate}
for further details and performance aspects of Nergiz et al.'s
database model. 

This paper investigates whether it is possible to learn a decision tree
in the above database model with minimal client resource usage.
We assume that the client has limited storage 
and processing resources, just enough to make small refinements on 
any model learned at the server. The remainder of this paper 
refers to the server of Nergiz et al.'s client-server database model as a cloud 
database server (CDBS). Terms ``CDBS'' and ``server'' are used interchangeably.

\subsection{Related Work}
There have been studies in how to mine anonymized data.
Fung et al. \cite{FungTDS} give a top-down 
specialization method (TDS) for anonymization so that the anonymized data 
allows accurate decision trees. Chui et al. \cite{ChuiItemSet} and Leung et al. 
\cite{LeungPattern} address the frequent itemset mining problem. Ngai et al. 
proposes a clustering algorithm that handles the anonymized data as 
uncertain information \cite{NgaiClustering}.
Kriegel et al. propose
a hierarchical density based clustering method using fuzzy objects \cite{HierarClustering}. 
Xiao et al. discusses the problem of distance calculation for uncertain objects 
\cite{DistanceUncertain}. Nearest Neighbor classification using generalized data is 
investigated by Martin \cite{MartinKnn}. Zhang et al. studies Naive Bayes 
using partially specified data \cite{NBCtaxonomy}. Dowd et al. studies decision 
tree classifier with random substitutions \cite{RandSubsTree}. Kantarcioglu 
et al. proposes a new Support Vector classification and Nearest Neighbor 
classification that deals with anonymized data \cite{kAnonSvm}. They extend
the generalization based anonymization scheme to keep all necessary 
statistics and use these statistics to build effective classifiers. Freidman et al.
investigate learning C4.5 decision trees from datasets that satisfy differential
privacy \cite{TreeDiffPriv}. Jagannathan et al. propose a tree classifier based
on random forests that is built from differentially private data \cite{RandForDiffPriv}.

These methods are all based on generalization/suppression or noise
addition techniques. Anatomized data provides additional detail that 
has the potential to improve learning, but also additional uncertainty that 
must be dealt with.  In addition, the database model of \cite{OutsourceUpdate,
ErhanIFIP} provides exact matching information that can be extracted 
only at the client; providing additional opportunities to improve learning.

Another approach related to mining anatomized data in the client-server
database model is vertically
partitioned data mining (e.g., \cite{CommunicationEffTree,ppdt}).
Vertically partitioned
data mining makes strong assumptions about data partitions
and what data can be shared, and typically assume that each
party holding data has significant computational power.
For example, some decision tree techniques assume that two tuples
of two partitions can be linked directly to each other, that the tuples are
ordered in the same way; and that the class labels are known for both 
partitions.
In one sense, our problem is easier, as one party (client) can see all
the data.  However, we assume this party (again client) does not 
have the resources to even store all the data,
much less to build the tree. %-- adding a new challenge.

\subsection{Contributions}
The model of \cite{OutsourceUpdate} would allow a client to construct
an exact decision tree by downloading the identifier
and sensitive tables of a dataset (or data table), decrypting
the join key, and joining the identifying and sensitive information
(rebuilding the original dataset). The remainder of this paper will use
terms dataset and data table interchangeably.
The goal of this paper is to reduce resource requirements at
the client, by doing as much work as possible at the server.
The challenge is that the server is not allowed to know exact
mappings between identifying and sensitive information (and
server can't know without background knowledge these mappings 
due to $l$-diversity of anatomization scheme).

We assume that the class attribute is not the sensitive attribute,
as effectively learning this at the server implies a de-facto
violation of $l$-diversity.  However, predicting values that fall
into identifying information (e.g., predicting geographic location
given other demographics and a disease) does not necessarily face this
difficulty.  We thus perform processing involving identifying
attributes at the server.  The client then needs only do the
necessary processing to utilize the sensitive attribute to complete
the tree.

The key novelty is that this final processing is:
\begin{enumerate}
\item Performed only when a prediction would make use of that information, and
\item The cost is amortized over many predictions.
\end{enumerate}
The resource requirements (including memory, CPU cost, and communications)
remain relatively low per prediction, thus allowing this to be
performed on lightweight clients that are assumed to have limited
storage capabilities.

We first 
give a set of definitions and define a prediction task for anatomized data. We then detail our 
method, along with cost and privacy discussions. 
%We also discuss another
%alternative based on homomorphic encryption.
Section \ref{sec:exper} shows performance with experiments
on various UCI datasets. The
performance evaluation includes execution time, memory requirement
and classification accuracy. We conclude the paper with a brief summary
and future directions.

\section{Definitions and Prediction Task}
\label{sec:def}

We give a set of definitions that are required to explain our work,
and use them to define the prediction task for anatomized data.

\begin{definition}
A dataset ${D}$ is called person specific dataset for population ${P}$
if each tuple ${t \in D}$ belongs to a unique individual ${p \in P}$.
\end{definition}

The person specific dataset ${D}$ has ${A_1, \cdots, A_d}$ identifier attributes and a sensitive 
attribute ${A_s}$. % throughout this paper.
We use ${D.A_i}$ to refer to the $i$th 
attribute of a person specific dataset ${D}$. Similarly, we use ${t.A_i}$ to refer
to the $i$th attribute of a tuple in a person specific dataset ${D}$ (either ${1 \leq i \leq d}$ or
${i=s}$ holds). 

\begin{definition}
A group ${G_j}$ is a subset of tuples in dataset ${D}$ such that ${D=\cup_{j=1}^{m} G_j}$, 
and for any pair ${(G_{j1},G_{j2})}$ where ${1 \leq i \neq j \leq m}$,
${G_{j1} \cap G_{j2}= \emptyset }$.
\end{definition}

\begin{definition}\label{defn:ldiverse}
A set of groups is said to be l-diverse if and only if for all groups ${G_j}$
${\forall v \in \pi_{A_s } G_j, \frac{freq(v,G_j )} {\vert G_j \vert} \leq \frac{1}{l}}$ where ${A_s}$ is
the sensitive attribute in ${D}$, ${freq(v,G_j )}$ is the frequency of ${v}$ in ${G_j}$ and ${|G_j |}$ is
the number of tuples in ${G_j}$.
\end{definition}

We define anatomy as in \cite{ErhanIFIP,OutsourceUpdate}.

\begin{definition}
Given a person specific dataset ${D}$ partitioned in m groups using l-diversity
without generalization, anatomy produces a identifier table ${IT}$ and 
a sensitive table ${ST}$ as follows. ${IT}$ has schema
\begin{center}
${(A_1,...,A_d,GID,ESEQ)}$
\end{center}
where ${A_i \in I_T}$ for ${1 \leq i \leq d= \vert I_T \vert}$, ${I_T}$  is the set of identifying 
attributes in ${D}$, ${GID}$ is the group id of the group and ${ESEQ}$ is the encryption of a 
unique sequence number, ${SEQ}$. For each ${G_j \in D}$ and each tuple ${t \in G_j}$ 
(with sequence number ${s}$), ${IT}$ has a tuple of the form:
\begin{center}
${(t.A_1,...,t.A_d,j,E_k (salt,s))}$
\end{center}
The ST has schema
\begin{center}
${(SEQ,GID,A_s)}$
\end{center}
where ${A_s}$ is the sensitive attribute in ${D}$, ${GID}$ is the group id of the group 
and ${SEQ}$ is a (unique but randomly ordered)
sequence number for that tuple in ${ST}$, used as an input 
for the ${ESEQ}$ of the corresponding tuple in ${IT}$. For each ${G_j \in D}$ and each 
tuple ${t \in G_j}$, ${ST}$ has a tuple of the form:
\begin{center}
${(s,j,t.A_s)}$
\end{center}
\end{definition}

The reason why the anatomy definition \cite{XiaoAnatomy} is extended 
in \cite{ErhanIFIP,OutsourceUpdate} and in this paper is that
CDBS operations like selection, insertion, deletion etc. require the 
consideration of the (encrypted where necessary) sequence number.
The
sequence number is in fact used to rebuild the original mapping between the identifying
attributes and the sensitive attributes of a given tuple (cf. Section \ref{sec:intro}). The
sequence numbers and encrypted sequence numbers are created during the 
splitting phase of an outsourced table (cf. Section \ref{sec:intro}). Here is a
brief explanation of how this mechanism works in several database operations:

Firstly, $ESEQ$ field values in the identifying table are decrypted by the client to 
find the true sequence number of given tuples. Secondly, the true sequence 
number values in the identifying table is equi-joined with the sequence number $SEQ$ 
field values in the sensitive table. The database has some sequence numbers in 
the sensitive table $ST$ which are same as the true sequence numbers
 (this is guaranteed in the initial splitting phase, see 
Section \ref{sec:intro}). The original tuples are eventually rebuilt after this equi-join 
operation. The reader is directed to \cite{ErhanIFIP,OutsourceUpdate} for further 
issues about CDBS design. In this paper, we will use sequence number values to 
do client-side refinement of the decision tree. 

\begin{definition}\label{defn:simplediverse}
An l-diverse group G is said to have a simple l-diverse distribution if 
\begin{center}
${\forall p \in G}$ and ${\forall v \in \pi_{A_s } G_j, P(p.A_s=v) = \frac{(freq(v,G_j))} {\vert G_j \vert} }$
\end{center}
where ${p}$ denotes an individual and ${A_s}$, ${freq(*)}$ are as in Definition \ref{defn:ldiverse}.
\end{definition}

\begin{definition}
A person specific dataset ${D}$ has \emph{simple l-diverse distribution} if for 
every group ${G_i \in D}$ where ${1 \leq i \leq m}$ holds, ${G_i}$  has a simple
l-diverse distribution according to Definition \ref{defn:simplediverse}.
\end{definition}

This paper focuses on collaborative decision tree learning between the client and
CDBS. Next,
we define the principal components of the collaborative decision tree.

\begin{definition}
Given two partitions ${IT}$ and ${ST}$ of a person specific dataset ${D}$, 
a \emph{base decision tree (BDT)} is a decision tree classifier that is built from ${A_1,\cdots,A_d}$
attributes ${(I_T)}$. Given a BDT that has leaves $Y$; every leaf ${y \in Y}$ has tuples in the 
following format:
\begin{center}
${(t.A_1, \cdots ,t.A_d,j,E_k (salt,s))}$
\end{center}
${j,E_k (salt,s)}$ fields are ignored when the BDT is learned.
Note that the
BDT maintains the simple $l$-diverse distribution which is the privacy constraint of CDBS. 
Given a leaf ${y \in Y}$, each tuple ${t \in y}$ is equally likely to match${l}$
different values of ${A_s}$. The privacy issues with the base decision tree are discussed 
in the next section.
\end{definition}

\begin{definition}
Given a base decision tree ${BDT}$, the leaves ${Y \in BDT}$; a leaf ${y \in Y}$ is called a
\emph{refined leaf} if and only if it \emph{points} to a sub-tree rooted at ${y}$ that 
is encrypted (using symmetric key encryption).
\end{definition}

We will elaborate in the privacy discussion of the next section how and why a base decision
tree has encrypted sub-trees.
%For now, we continue giving the last definition.

\begin{definition}
Given a base decision tree ${BDT}$, the leaves ${Y \in BDT}$; a leaf ${y \in Y}$ is called 
\emph{unrefined leaf} if and only if it \emph{doesn't point} to a sub-tree rooted at ${y}$.
\end{definition}

In the remainder of this paper, ${Y}$ is used to note the leaves of a decision tree.

We finish the discussion in this section by defining the prediction task of our classifier. 
The prediction task is to predict an identifying attribute (${A_i \in I_T}$) given other
identifying attributes (${A_j}$  satisfying ${1 \leq i \neq j \leq d}$) and the sensitive 
attribute (${A_s}$).

A real life example would be learning from
patient health records.  The U.S. Healthcare laws protect
individually identifiable health information \cite{HIPAAfinal}.
An anatomized database, while it contains individually identifiable
information, does not link this directly to the (sensitive) health
information.
% personal identifiable health information of outsourced patient records.
%Patient records outsourcing must preserve the privacy of a patient's
%personal identifiable information \cite{ref26} (HIPAA regulation).
Suppose we have public directory information as identifying
attributes including \texttt{zipcode}, \texttt{gender}, \texttt{age},
as well as specific identifiers such as name (which are presumably
uninteresting for building a decision tree.)  We also have a
sensitive table consisting of \texttt{diagnosis}.  
Suppose the goal was to identify locations where people with
particular backgrounds (including health conditions) live,
through learning a decision tree with class value \texttt{location}.
While this seems a somewhat contrived example, it would likely
have identified the infamous Love Canal district\cite{LoveCanal}
as an unusually common location for those with the diagnosis
\emph{birth defect}.

\section{Anatomization Decision Trees}

We investigate three possible approaches to build decision trees from
anatomized data in Nergiz et al.'s client-server database model. We call
this family of decision trees as \emph{anatomization decision trees}, 
because they are all built from anatomized data.
Each approach differs according to how ${IT}$ and ${ST}$ tables are used.

In the discussion of anatomization decision trees, algorithms that require
encryption are assumed to use AES-128 symmetric key encryption.
We first discuss two baseline approaches:  pure-client side 
learning, and pure server-side learning. Section \ref{sec:ourmethod} explains the
collaborative decision tree learning. We finally 
compare collaborative decision tree learning with the 
baselines in Section \ref{sec:compare}. 
 
\subsection{Baseline Approaches: CDBS Learning and Client Na\"{i}ve Learning}
\label{sec:baselines}
 
The baseline approaches compose two different scenarios for learning. 
The first scenario is that the CDBS learns a base decision tree using
only the ${IT}$ 
partition of dataset ${D}$ (cf. Section \ref{sec:def}). 
The second scenario is that client learns a decision
tree using the person specific dataset ${D}$ retrieved from the CDBS.
We call the first learning scenario CDBS learning
and the second learning scenario client na\"{i}ve learning. 

In client na\"{i}ve learning,
the client retrieves the person specific dataset $D$ by getting
all the tuples in the respective $IT$ and $ST$ partitions of $D$. $IT$ and
$ST$ partitions are joined together using the join key and the non-anatomized set of 
records in dataset ${D}$ is obtained (cf. Sections \ref{sec:intro} and \ref{sec:def}). The client learns
decision tree from this non-anatomized set of records in dataset $D$. In other words, the client
rebuilds from its outsourced data the non-anatomized version (original version) 
of dataset $D$ and learns the decision tree itself without any attribute 
uncertainty. Client na\"{i}ve learning brings up some vital issues and respective
solutions below:

The privacy aspect of the decision tree is an issue. The client learns
the decision tree from the non-anatomized version of ${D}$, so the identifying
attributes and the sensitive attribute are associated in the decision tree.
The client has to store this model at the CDBS (limited storage resource
assumption, see Section \ref{sec:intro}), but since
the model is built from non-anatomized dataset, it may reveal information
violating privacy constraints (particularly in conjunction with
the anatomized data).  The solution is to encrypt the decision tree and to store
the encrypted tree at CDBS.

Encrypting the decision tree raises the issue of how to make an inference of the class label
for a new instance. The CDBS cannot make the inferences because the stored
tree is encrypted. This leads to the following approach:
each time the class label of a new tuple is predicted,
the client retrieves the encrypted decision tree, 
decrypts the encrypted decision tree and makes
the prediction. This makes the inference phase quite expensive for the
client in client na\"{i}ve learning. 

The other extreme is CDBS learning: the server constructs a base
decision tree using only the identifying information. 
Base decision tree makes the prediction operation very
easy and low cost for the client. Given a predicted 
tuple ${t_p}$ and a base decision tree for predicting class label ${A_c \in I_T}$;  
the client just deletes the attribute field ${t_p.A_s}$ (an ${O(1)}$ time 
operation). Then, it sends the rest of the tuple ${t_p}$ to the CDBS 
and CDBS makes the tree inference. The client has ${O(1)}$ inference cost
and network overhead.
This low cost for client is the main advantage of CDBS learning.
%(no learning cost and ${O(1)}$ inference cost 
%for client with small network overhead).
The drawback of CDBS learning is that the base decision tree ignores the
sensitive attribute. This is not a big drawback
if ${A_s}$  is a weak predictor of the class label. On the other hand, the base
decision tree's accuracy is likely to be low if ${A_s}$  is a strong 
predictor of the class label.

\begin{figure}
\begin{center}
\includegraphics[width=50mm,height=35mm]{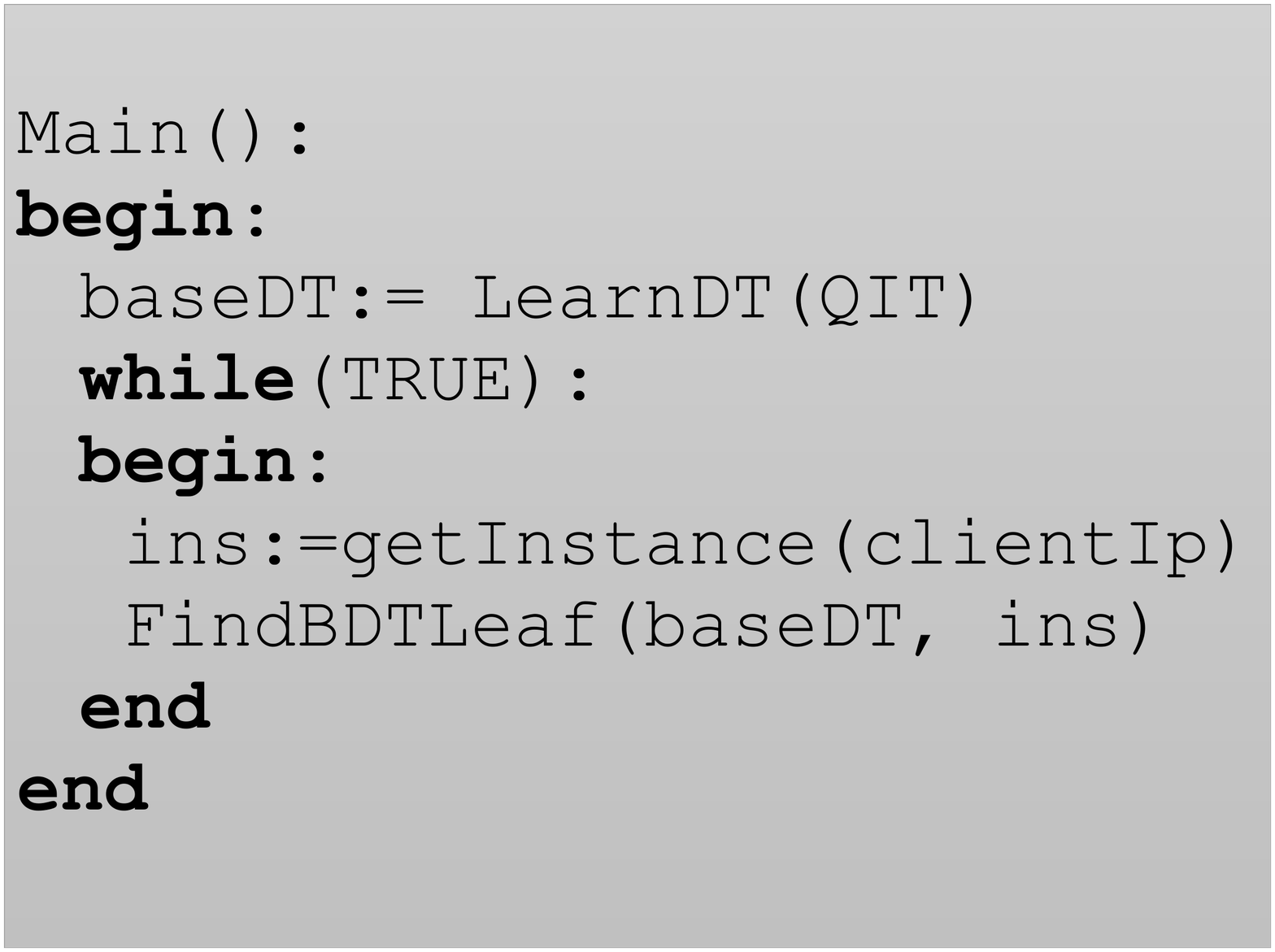}
\caption{Collaborative Learning Main Function (Called by Server)}
\label{fig1}
\end{center}
\end{figure}

\subsection{Collaborative Decision Tree Learning}\label{sec:ourmethod}

One way to mitigate the drawbacks of CDBS learning and client na\"{i}ve learning 
is the usage of grouping statistics among anatomy groups of the person specific 
dataset ${D}$. Grouping statistics are the distribution of ${A_s}$ 
among anatomy groups ${G_i}$ provided that ${1 \leq i \leq m}$ and ${\cup_m G_i =D}$
hold. Grouping statistics might have the potential to determine an interesting
decision tree split for a class label ${A_i \in I_T}$. However, to avoid
correlation-based attacks, we assume the insertion
operation of CDBS is designed to create groups randomly while satisfying
the  simple $l$-diverse distribution \cite{OutsourceUpdate}.
It is unlikely to
find an interesting pattern
for split using group statistics.

This paper proposes a new collaborative decision tree algorithm involving CDBS and
client collaboration. Given a dataset ${D}$ partitioned on ${IT}$ and ${ST}$, $IT$ and $ST$
partitions stored on CDBS, a class label
${A_i \in I_T}$; CDBS initiates building the collaborative decision tree. First, the server builds a base
decision tree.
%using ${A_1,…,A_(i-1),A_(i+1),…,A_d}$ in ${IT}$ table.
Then, the base
decision tree leaves ${Y}$ are improved by the client. For every leaf ${y \in Y}$, the
improvement is a sub-tree learned by the client. The client uses the tuples in leaf ${y}$
to learn an improved sub-tree. 
%The splits from root to a certain depth are
%made by the CDBS. The splits from the former depth to the leaves are made by
%the client. 
%A distributed decision tree classifier is eventually built.
Figures \ref{fig1}, \ref{fig2} and \ref{fig3} give pseudo code of the collaborative 
decision tree learning process.

Figure \ref{fig1} shows the \texttt{Main()} function that is called by CDBS to initiate 
collaborative learning.  The
\texttt{LearnDT(IT)} function call builds the base decision tree from 
the $IT$ partition.
The client improvements on the base decision tree are made
on the fly when doing predictions.
To make a prediction, a tuple
${(ins.A_1, \cdots , }$ ${ins.A_{i-1}, ins.A_{i+1}, \cdots , ins.A_d)}$
is sent to the CDBS.
The \texttt{getInstance()} function call % which is executed in infinite loop,
receives the sent tuple \texttt{ins} (cf. Figure \ref{fig1}). \texttt{FindBDTLeaf(baseDT, ins)} 
function is called by CDBS for every predicted tuple \texttt{ins} 
once the predicted tuple \texttt{ins} is received by \texttt{getInstance()} calls
%in the infinite loop 
(cf. Figure \ref{fig1}). 
%\texttt{FindBDTLeaf}
%function is the place where client gets into collaboration to improve
%the base decision tree.

\begin{figure}
\begin{center}
\includegraphics[width=75mm,height=60mm]{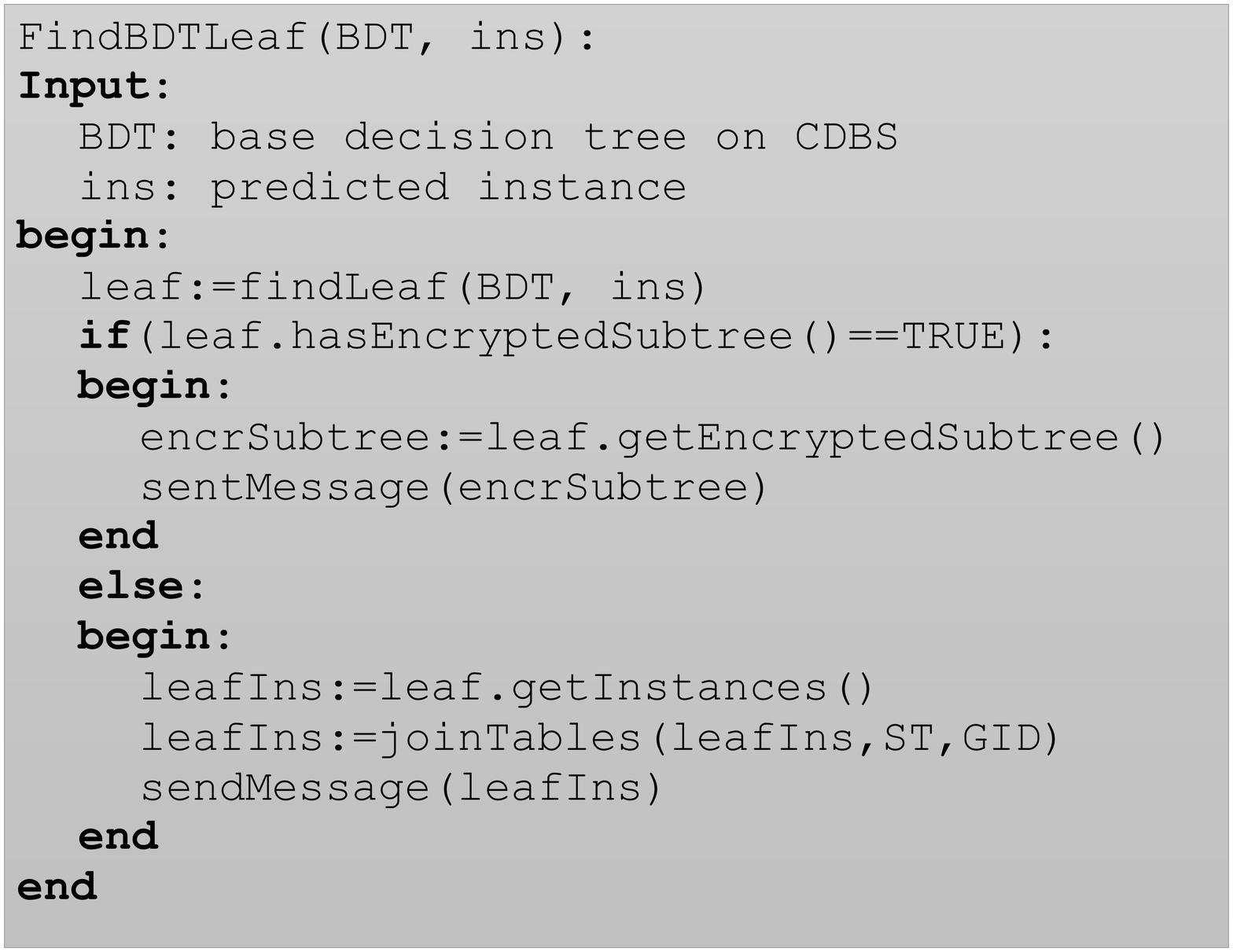}
\caption{Collaborative Learning \texttt{FindBDTLeaf} Function (Called by Server)}
\label{fig2}
\end{center}
\end{figure}

The
\texttt{FindBDTLeaf} function essentially finds the appropriate base decision tree leaf
${y \in Y}$ for the client to complete (if not already done) and sends leaf $y$ to the client (cf. Figure \ref{fig2}). Given tuple \texttt{ins} and the base decision tree 
\texttt{baseDT}, the \texttt{findLeaf(BDT,ins)} function call
finds the correct leaf ${y \in Y}$ using a usual decision tree inference with attribute values
${(ins.A_1, \cdots , ins.A_{i-1},}$ ${ins.A_{i+1}, \cdots , ins.A_d)}$.
%There is no possible privacy violation in this step since the client
%sends the non-class identifying attributes of tuples ${t_p}$. Non-class
%identifying attributes don’t include any information about sensitive
%attribute. 
Then, \texttt{FindBDTLeaf} verifies if ${y}$ points to an
encrypted sub-tree that was previously learned by a client (\texttt{if} statement
in pseudocode).
(The subtree is encrypted to ensure privacy constraints are satisfied,
this will be discussed in more detail later.)
If ${y}$ points to an encrypted sub-tree, it sends the client
the encrypted sub-tree as a response message (\texttt{sentMessage(encrSubTree)} 
function call). Otherwise, CDBS
sends the tuples belonging to ${y}$ by function call \texttt{sentMessage(leafIns)}.
%\texttt{FindBDTLeaf} function finishes its execution after sending the required message to the client.

All
tuples ${t \in y}$ are partitioned across ${IT}$ and ${ST}$ on CDBS. 
%(unless
%pruning dictates that no further division of a subtree is necessary, e.g.,
%all tuples in the subtree are of the same class.)
Reconstructing the original tuples $t \in y$ is done on the client side (\texttt{Inference()} function call, cf. Figure \ref{fig3}, explained later).
The
\texttt{joinTables(leafIns, ST, GID)} function call prepares the message \texttt{leafIns} including
the following tuples: \texttt{leafIns} include every tuple ${t \in y}$ with all combinations
of $l$ sensitive attributes. In other words function \texttt{joinTables()} matches every tuple $t \in y$ with all $l$
potential sensitive attributes using the $GID$ field as the join key (a group-level join, giving all possible tuples given the l-diverse dataset, not just the
true matching values).
The explanation of \texttt{Inference()} function will clarify why this
join operation is done (cf. Figure \ref{fig3}).

\begin{figure}
\begin{center}
\includegraphics[width=75mm,height=90mm]{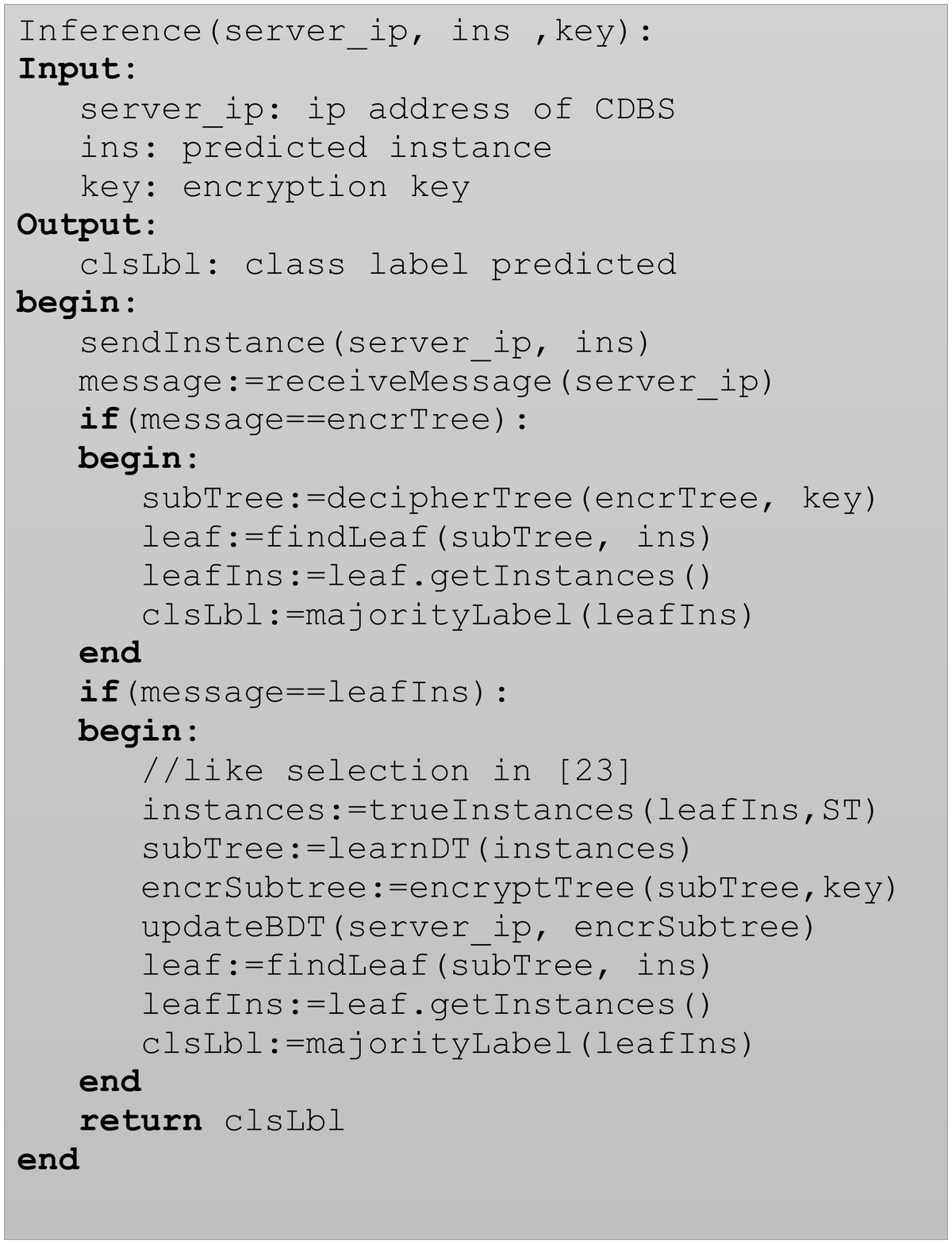}
\caption{Collaborative Learning \texttt{Inference} Function (Called by Client)}
\label{fig3}
\end{center}
\end{figure}

A client calls the \texttt{Inference()} function to make predictions or
on the fly improvements (cf. Figure \ref{fig3}). The client sends 
a predicted tuple \texttt{ins} to the CDBS using the
\texttt{sendInstance()} function. It receives the \texttt{message}
sent by function \texttt{FindBDTLeaf()} using \texttt{receiveMessage()}.
As mentioned in the \texttt{FindBDTLeaf()} discussion, the \texttt{message} can 
be either an encrypted sub-tree (\texttt{encrTree}) or tuples $t$ of leaf $y$ (\texttt{leafIns}).
\texttt{leafIns} is in fact all tuples $t \in y$ such that every tuple $t$
is matched with all $l$ potential sensitive attribute values.
%(e.g. see previous paragraph).
%Once the client receives an incoming message from the CDBS, 
%it identifies the content of the message (cf. Figure \ref{fig3}). 
If the \texttt{message} is an encrypted sub-tree (\texttt{encrTree}), 
\texttt{Inference()} just decrypts the encrypted
sub-tree using \texttt{decipherTree()} function, finds the appropriate
leaf $y^\prime$ as in regular tree inference using \texttt{findLeaf()} function and 
predicts the class label of tuple \texttt{ins} by taking majority of the
class labels in $y^\prime$ using \texttt{majorityLabel()} function.  If 
\texttt{message} is \texttt{leafIns}, the \texttt{Inference()} function
reconstructs the non-anatomized leaf instances (\texttt{instances}) from \texttt{leafIns}: it
decrypts the encrypted sequence numbers $ESEQ$ in identifying
attributes of \texttt{leafIns} to find true sequence numbers 
and eliminates the tuples $t$ of \texttt{leafIns} 
that don't have the sensitive table sequence numbers $SEQ$ the same as
true sequence numbers. \texttt{trueInstances()} function does the former operation. 
We discussed earlier in Section \ref{sec:def} how
this sequence number fields are designed and how they are used in Nergiz et al.'s
client server database model \cite{ErhanIFIP, OutsourceUpdate}.  \texttt{Inference()} 
learns a sub decision tree (\texttt{subTree}) from \texttt{instances} 
using \texttt{learnDT()} function, encrypts \texttt{subTree} (\texttt{encryptTree()} function
call) and sends \texttt{subTree} back to CDBS for updating as we discussed earlier in \texttt{FindBDTLeaf()}.
Finally, \texttt{Inference()} finds the appropriate
leaf $y^\prime$ as in regular tree inference using \texttt{findLeaf()} function and 
predicts the class label of tuple \texttt{ins} by taking majority of the
class labels in $y^\prime$ using \texttt{majorityLabel()} function.

We finish the discussion in this subsection by considering
the encryption of the client's improved sub-trees.
Encryption is necessary due to privacy concerns.
Once the sub-tree is built, % in \texttt{Inference()} function,
it has leaves having  tuples
in ${(t.A_1, \cdots ,t.A_d,t.A_s)}$ format. Moreover, the sub-tree also
has splits with the true values of ${A_s}$. Thus, whole sub-tree should
be encrypted so that no additional information regarding the
correlation between tuples and sensitive values is provided to the CDBS
when the sub-tree is stored.
The symmetric key encryption of sub-tree preserves the privacy of
person specific dataset $D$ according to following theorem.

\begin{theorem}
(Privacy Preserving Learning): Given a person specific dataset ${D}$ on
CDBS having a simple $l$-diverse distribution, the collaborative decision tree
learning preserves the simple $l$-diverse distribution of ${D}$.
\end{theorem}

\begin{proof}
Given a snapshot ${IT(i)}$ and ${ST(i)}$ of ${D}$ on CDBS at time $i$, 
suppose that a collaborative decision tree is learned using the
algorithm in Figures \ref{fig1}--\ref{fig3}.
CDBS learns the associations within the
identifying attributes (Figure \ref{fig1}). The CDBS is allowed to learn such
associations according to its privacy
definition \cite{ErhanIFIP} \cite{OutsourceUpdate}. 
These associations do not provide any background information linking
identifying attributes and sensitive attribute. The Client learns all associations 
within identifying attributes and the sensitive attribute (cf. Figure \ref{fig3}), 
but no additional information is provided back to the server (except in encrypted form.)
The simple $l$-diverse distribution is maintained for snapshots ${IT(i)}$, ${ST(i)}$ 
at ith time. The inference operation at jth time (${1 \leq i \leq j \leq \infty}$) is done
with CDBS and client. The base decision tree has only associations between
identifying attributes, so the CDBS part of inference is safe. The
encrypted sub-trees are used on client side of inference, so assuming
semantically secure encryption this gives no further information to
the CDBS.
%inference does not provide additional information to CDBS about
%the association between sensitive attribute and identifying attributes.
Throughout predictions, some base decision tree leaves might never be visited,
but this only discloses that certain \emph{identifying} information
has not come up for prediction.
This is not a disclosure case about sensitive information.
%%% The following isn't quite true - see above.
%The unvisited base decision tree leaves mean that the identifying
%attributes are not good predictors for the tuples in these leaves
%(or even for the classification task). 
CDBS cannot know whether the sensitive attribute has a better information for the 
unvisited leaves' tuples as sensitive attributes' positive/negative correlation or 
uncorrelation with identifying attributes would not be known
without any background knowledge. The encrypted subtrees of visited leaves
do not let CDBS learn this kind of correlation information.
Consequently, collaborative decision tree learning preserves 
the simple l-diverse distribution in inference at time ${j}$.
%Q.E.D.   % The begin/end proof should insert the publisher's preferred Q.E.D.
\end{proof}

Our privacy analysis doesn't include the case where there are multiple clients that
outsource different data. % in anatomization format.
The reader might be concerned that if there
are multiple clients who outsource their data in anatomization format, some clients' identifying
information can be other clients' sensitive information. This cross correlation might create
privacy concerns about the whole client-server database model. This is an issue dealt in 
the original Nergiz et al.'s client-server database model  \cite{OutsourceUpdate,ErhanIFIP}. 
%We don't include here again the same privacy discussion as this paper
%does not focus on the  anatomization based client/server database model
%and its privacy challenges.
%Decision  tree learning techniques are used in an anatomization based
%client-server database model that ensures $l$-diversity in the presence of
%multiple clients. We also assume that the connection between multiple
%clients are not compromised during anatomized decision tree learning.
Since we have shown that our decision tree learning reveals no additional
information, it cannot result in privacy violations provided the underlying
database does not violate privacy constraints.

\subsection{Cost Discussion}\label{sec:compare}

It is difficult to estimate a practical average cost savings 
for collaborative decision tree learning, as it depends
not only on the data, but on the client's memory, processing,
and communication resources.

If the identifying attributes are bad predictors, the base decision
tree will be complex. A complex decision tree model means that
it has many splits yielding small leaves, requiring only a
small amount of client memory and computation during the
prediction phase.
The client needs to build more complex sub-trees to
achieve good prediction.
However, small leaves means that the cost per sub-tree will
be small, even though the total cost (amortized over many
predictions) is high.
If the identifying attributes are good predictors, the
base decision tree will be simple.  A tree having few
splits will yield fewer, but larger leaves.
This increases the cost for the first prediction, but
also increases the likelihood that future predictions will
hit an already completed subtree.

In either case, the client requires less memory than client na\"{i}ve
learning, as it needs to hold at most the tuples that reach a leaf,
along with all possible sensitive values for those tuples. 
The next theorem
gives an upper bound for collaborative decision tree learning cost
under reasonable assumptions and feasible alternatives.

\begin{theorem}
(Cost Upper Bound) Given a person specific
dataset ${D}$, ${IT}$ and ${ST}$ partitions on CDBS storing ${D}$,
similar number of splits in base decision tree, client improvement
sub-trees and client na\"{i}ve learning tree; the client cost of collaborative
decision tree learning cannot be bigger than client na\"{i}ve learning's cost.
\end{theorem}

\begin{proof}
Assume that ${D}$ is ${O(n \times m)}$ where ${n}$ is the
number of rows and ${m}$ is the number of attributes. Let us also assume
that the client cost of collaborative decision tree learning is bigger than
the client na\"{i}ve learning cost. Given a base decision tree with ${O(m)}$
leaves, depth, and split (worst case tree type); the total cost
of base decision tree learning is ${O(nm^2)}$. This cost is excluded in
client cost since it is not executed by the client. Given a client sub-tree on
the $i$th base decision tree leaf (${n_i}$ is number of rows) with ${O(m)}$ splits, 
leaves and depth (same as base decision tree); the cost of learning a
sub-tree is ${O(n_i m^2)}$. The total cost of collaborative decision tree
learning is ${O(\sum_i n_i m^2)}$  which is ${O(nm^2)}$ (${1 \leq i \leq m}$).
Given client na\"{i}ve learning's resulting model with ${O(m)}$ splits,
leaves and depth (same as base decision tree); the total client
cost is ${O(nm^2)}$. Client na\"{i}ve learning has a client cost that is
equal to the collaborative decision tree learning's client cost. This
result contradicts to the basic assumption. Thus, collaborative
decision tree learning cost cannot be bigger than client na\"{i}ve
learning cost by contradiction.
\end{proof}

The same theorem can be developed for encryption, decryption and network
costs as well (intrinsic). The number of splits assumption in the cost theorem is a
worst case assumption for cost calculation. In practice, the sub-tree on the ith
leaf of the base decision tree will be much simpler (have fewer splits)
than both the base decision and client na\"{i}ve learning trees.
Client sub-trees would be learned from less data and
less data yields the simple subtrees having fewer splits. Fewer splits mean 
low training, inference, encryption/decryption and network costs for client 
in collaborative decision tree learning.

\section{Experiments}\label{sec:exper}

We now compare our collaborative decision tree learning (cdtl) with
the client na\"{i}ve (cnl) and CDBS learning (cdbsl) on four
%This section presents the experiment results for three decision
%tree learning algorithms: CDBS learning (cdbsl), client na\"{i}ve learning
%(cnl), and distributed decision tree learning (ddtl). We use four
datasets from the UCI collection: adult, vote, autos and Australian
credit.  We evaluate two things:  The classification accuracy
(presumably better than clbs, but worse than client na\"{i}ve),
and the cost for those performance improvements. 
%We don't
%include in our experiments homomorphic encryption based alternative
%as it is infeasable (Section \ref{sec:homomorphic}).

The Adult dataset is composed of US census data. An individual's
income is the class attribute (more than 50K vs less than 50K).
It has 48842 tuples where each tuple has 15 attributes.
The Vote dataset
contains 485 tuples where each tuple has 16 binary attributes. 
An attribute is the vote of a senator in a session.
A senator's party affiliation is the class attribute
(democrat vs republican).
The Autos datasets contains 205 tuples where
each tuple has 26 attributes.
%about characteristics of a car, its risk
%rating and its normalized losses relative to other cars. 
The class label
is either symboling (risk rating) or price of a car. Autos dataset's
continuous attributes are discretized in our experiments and a binary
class label is created from price of a car (low price vs. high price).
The Australian credit dataset has 690 tuples where each tuple has 16 
attributes. 
%The dataset is about the approval of credit card
%applications. All attribute names and values are meaningless symbols
%to protect confidentiality of the data. 
The class attribute is whether
a credit application is approved or not (${+}$ vs. ${-}$).
We chose these four because they are large enough to demonstrate
performance differences, and reasonably challenging
for decision tree learning.
The reader is advised
to visit \cite{UCIrepository} to learn more about the structure of the
datasets. 

Experiments 1 and 2 are made on auto and vote datasets respectively, 
because they can simulate a real world example of \emph{$l$-diverse} data.
While privacy may not be a real issue for this data, we give an
examples below that we feel are a reasonable simulation of a
privacy-sensitive scenario.  (True privacy-sensitive data is hard
to obtain and publish results on, precisely because it is privacy-sensitive.)
Experiment 1 uses attribute symboling (risk rating) as the sensitive
attribute and predicts the attribute price. 
%Symboling is the initial risk factor associated with a car's price.
%Actuarians assign the symboling scores of cars. 
Suppose that there is
a dealership outsourcing their cars' pricing data. Clients might not want
to buy some of their cars because of high risk factor (high insurance
premium), so $l$-diversity is preserved for symboling. 
%The privacy violation can affect the dealership's business profit.
Experiment 2 is done on the vote dataset. It sets physician
fee-freeze as the sensitive attribute. 
%Suppose that a senate has data about 
%voting records. 
Suppose that senate's voting records are outsourced to a cloud database. 
Voting records in public sessions are available and they can be used with
party affiliation to identify a senator's identity. Thus, the voting records
of a private session should be \emph{l-diverse}. 
%Given the party affiliation and the
%voting trend in public sessions of a senator, his/her identity can be
%determined. Then, this identification information can be used to find out
%what he voted for in the private session of assembly. As the senator's
%identity and his/her vote in private session is revealed, this is a privacy violation. 

Experiments 3 and 5 are made on the Australian credit dataset and adult
dataset.  Decision tree models fit well on these datasets.
These experiments show how the cdtl models behave when the underlying
data is convenient for decision tree learning. Experiment 3 uses Australian
credit dataset. It assigns attribute A9 as sensitive attribute and predicts
the class attribute. Experiment 4 uses the
adult dataset. It assigns attribute relationship as sensitive attribute and
predicts the class attribute. 
%Privacy community considers this dataset’s
%class attribute as sensitive attribute but we consider other attributes as
%sensitive in this paper. 
Relationship and A9 are the strongest predictors
of class labels in adult and Australian credit datasets. Experiment 4 and 5
measure together the effect of the sensitive attribute's predictor power
given a fixed dataset. In experiment 5, the adult dataset is used again with
sensitive attribute age that is a moderate predictor. 
Experiments 4 and 5 use the tuples which belong to individuals in non-private
work class. So, the experiments are made on a subsample of 14936 tuples. 

On each experiment, we apply cdbsl, cnl, cdtl with 10-fold cross validation. 
We measure accuracy, memory savings and execution time savings.
Weka J48 is used for decision tree learning with reduced error pruning.
20\% of the training sets are used for reduced error pruning. Given a
training/test pair of 10-fold cross validation, the prune set is chosen
randomly once so that cdbsl, cnl, cdtl models are compared within the
same model space (same pruning set and training set for each algorithm).
The experiments on vote and australian credit datasets learn binary
decision trees whereas the experiments on auto and adult datasets
learn non-binary decision trees. Experiments are done using a physically 
remote cloud server; and a laptop with Intel i5 processor and 4 GB RAM. 
Internet connection speed was 100 Mbps.

\begin{equation}
\label{eq4.1}
ets=\frac{cdtl \, execution \, time}{cnl \, execution \, time}
\end{equation}

\begin{equation}
\label{eq4.2}
ms=\frac{cdtl\,memory\,requirement }{cnl\,memory\,requirement}
\end{equation}

Equations \ref{eq4.1} and \ref{eq4.2} calculate \emph{execution time 
savings (ets)} and \emph{memory savings (ms)} respectively in the experiments. 
Savings are maximized as both formulas approach to zero.
In contrary, savings are minimized as both formulas approach 1.
Execution times (Eq. \ref{eq4.1}) include client's encryption/decryption, 
network and learning/inference costs on a training set and test set pair
(ith iteration of cross validation). In equation \ref{eq4.2}, the number 
of tuples in training set is the cnl's memory requirement whereas the 
number of tuples in the base decision tree's (BDT) biggest leaf is the 
cdtl's memory requirement. Since the client learns sub-trees on the fly from
the unrefined leaves, the biggest leaf is the memory upper bound.

Figure \ref{fig4} and Figure \ref{fig5} provides box plots showing 
elapsed time and memory savings
measurements on each of the 10 folds. Cdbsl is not shown, as
the client cost is 0.
The blue dots on the
boxplots show the mean time and memory savings.
Savings graphs exhibit visually 
the tradeoff between ms and ets. 
Given a dataset, if the memory savings are high, the time requirements
are low (as expected.) 
%Similarly, if the ms tends to be low; ets tends to be high.
It is expected 
since the high ms indicates a complex base decision 
tree model in cdtl. The client needs to do more improvement in cdtl since 
the leaf provided is already a bad predictor. In addition,  high ms and low ets 
can lead to complex decision 
trees that can produce overfitting. On the other hand,
low ms indicate a simple base decision tree in cdtl.
Little remains to be done at the client
since the leaf provided is already a good predictor. 

\begin{figure}
\includegraphics[width=85mm,height=100mm]{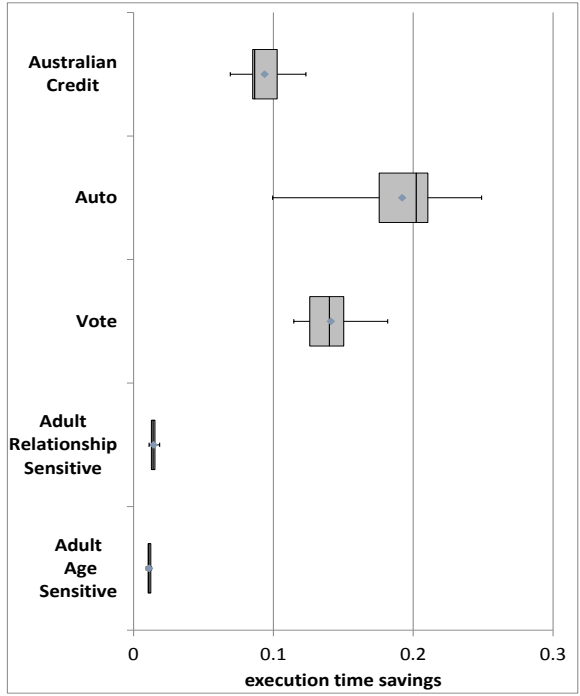}
\caption{Execution Time Savings Graph}
\label{fig4}
\end{figure}

\begin{figure}
\includegraphics[width=85mm,height=100mm]{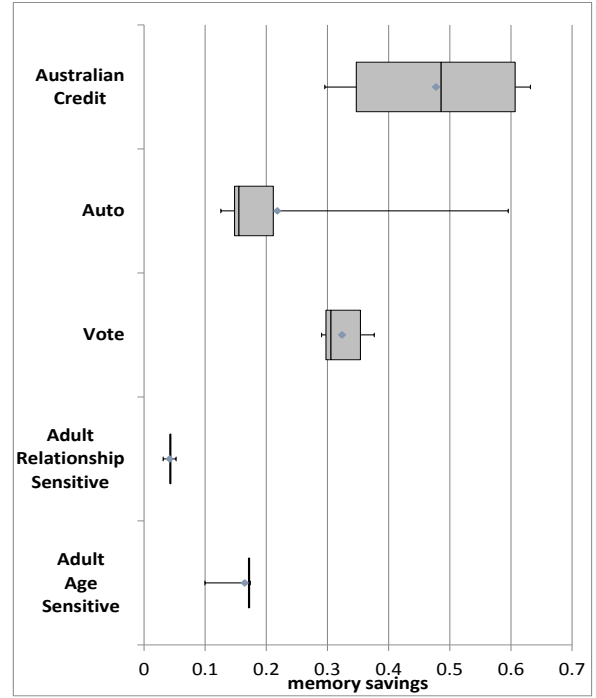}
\caption{Memory Savings Graph}
\label{fig5}
\end{figure}

\begin{figure}
\includegraphics[width=85mm,height=100mm]{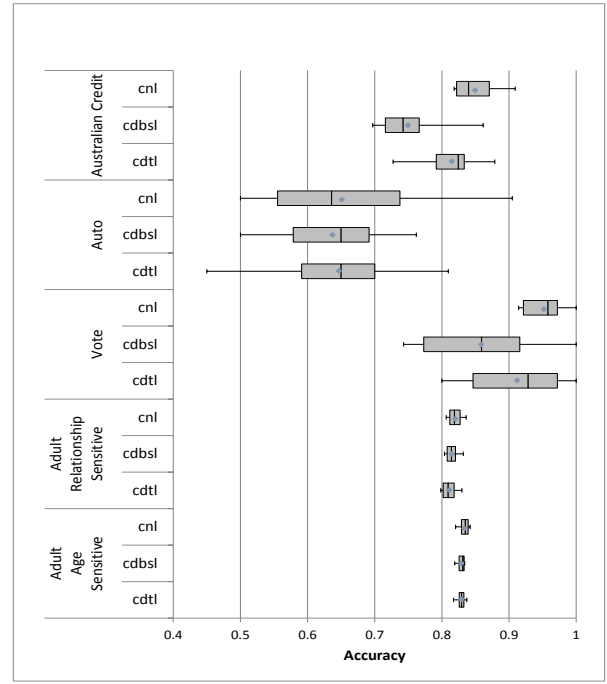}
\caption{Accuracy Graph}
\label{fig6}
\end{figure}

Figure \ref{fig6} provides measured prediction accuracies in 
the experiments. On average, the cnl decision trees have the best 
accuracy, as expected. Cdbsl decision trees have the worst accuracy and 
cdtl decision trees is generally somewhere between. 
%cnl and cdbsl decision trees.
This accuracy trend is 
expected since cdbsl gives a decision tree which is essentially
the same as the base decision tree in cdtl. So, cdtl's accuracy 
values show the effect of client improvements which are done 
to the cdbsl model. The exception is in experiment 4 (the Adult
dataset).  The sensitive attribute is not a very good predictor.
%Ddtl accuracy mitigation does not hold in
%experiment 4's accuracy values (adult dataset, sensitive 
%attribute relationship). Ddtl average accuracy is worse than
%cdbsl accuracy.
This exceptional case shows that the collaborative 
model becomes too complex after client improving (overfitting). 
Execution time savings and memory savings graphs justify this fact
(Figs. \ref{fig5} and \ref{fig6}). 
A possible solution to avoid cdtl overfitting is memory savings
threshold (threshold for BDT leaf size on CDBS). However, it is hard 
to define an exact threshold value. Given a fixed training set 
on CDBS, cdtl can be learned with various BDT thresholds. The 
threshold having the best accuracy on the pruning set can be 
chosen and client improvements can be applied on this model.

\section{Conclusion}

This paper proposes a decision tree learning method for
outsourced data in an anatomization scheme. A real world 
prediction task is defined for anatomization.  Collaborative
decision tree learning (cdtl), which uses cloud database server (CDBS) 
and client, is studied to achieve this prediction task. 
Cdtl is proven to preserve the privacy of data
provider. Cdtl is tested on 
various datasets and the results show that fairly accurate 
decision trees can be built whereas client's learning and inference
costs are reduced remarkably.

The next challenge is to extend this work in a new framework
such that there is no client processing while accurate decision trees
are learned. Another direction is to observe how the collaborative
decision tree learning algorithm performs relative to the
recent differential privacy decision trees.

\section{Acknowledgments}
We wish to thank to Dr. Qutaibah Malluhi and Dr. Ryan Riley 
for their helpful comments throughout the preparation of this work.
We also wish to thank to Qatar University for providing physical
facilities required for experimentation. This publication was made 
possible by NPRP grant 09-256-1-046 from the Qatar National 
Research Fund. The statements made herein are solely the 
responsibility of the authors.

\bibliographystyle{abbrv}
\bibliography{ref}  % sigproc.bib is the name of the Bibliography in this case
% ACM needs 'a single self-contained file'!
%\balancecolumns

%\subsection{Additional Authors}
%This section is inserted by \LaTeX; you do not insert it.
%You just add the names and information in the
%\texttt{{\char'134}additionalauthors} command at the start
%of the document.

% That's all folks!
\end{document}